\documentclass[sigconf]{acmart}

\AtBeginDocument{%
  \providecommand\BibTeX{{%
    \normalfont B\kern-0.5em{\scshape i\kern-0.25em b}\kern-0.8em\TeX}}}

\copyrightyear{2019}
\acmYear{2019}
\setcopyright{acmlicensed}
\acmConference[MileTS '19]{MileTS '19: 5th KDD Workshop on Mining and Learning from Time Series}{August 5th, 2019}{Anchorage, Alaska, USA}
\acmBooktitle{MileTS '19: 5th KDD Workshop on Mining and Learning from Time Series, August 5th, 2019, Anchorage, Alaska, USA}
\acmPrice{15.00}
\acmDOI{10.1145/1122445.1122456}
\acmISBN{978-1-4503-9999-9/18/06}

\usepackage{amsmath}
\usepackage{pifont}
\usepackage{amsthm}
\usepackage{amsfonts}
\usepackage{amssymb}
\usepackage{mathtools}
\usepackage{hyperref}
\usepackage{algorithm}
\usepackage{algorithmic}
\usepackage{xcolor}
\usepackage{enumitem}
\newcommand{\ms}{\overset{\rightarrow}{m}}
\newcommand{\mydots}{\hbox to 1em{.\hss.\hss.}}

\begin{document}

%
\title[A Robust Distance Metric]{A Formally Robust Time Series Distance Metric}

%

\author{Maximilian Toller}
\affiliation{%
  \institution{Know-Center GmbH}
  \streetaddress{Inffeldgasse 13}
  \city{Graz}
  \country{Austria}}
\email{mtoller@know-center.at}

\author{Bernhard C. Geiger}
\affiliation{%
	\institution{Know-Center GmbH}
	\streetaddress{Inffeldgasse 13}
	\city{Graz}
	\country{Austria}}
\email{bgeiger@know-center.at}
\email{geiger@ieee.org}

\author{Roman Kern}
\affiliation{%
  \institution{Graz University of Technology}
  \streetaddress{Inffeldgasse 13}
  \city{Graz}
  \country{Austria}}
\email{rkern@tugraz.at}

%

%
\begin{abstract}
Distance-based classification is among the most competitive classification methods for time series data.
The most critical component of distance-based classification is the selected distance function.
Past research has proposed various different distance metrics or measures dedicated to particular aspects of real-world time series data, yet there is an important aspect that has not been considered so far: \textit{Robustness against arbitrary data contamination}.
In this work, we propose a novel distance metric that is robust against arbitrarily ``bad'' contamination and has a worst-case computational complexity of $\mathcal{O}(n\log n)$.
We formally argue why our proposed metric is robust, and demonstrate in an empirical evaluation that the metric yields competitive classification accuracy when applied in k-Nearest Neighbor time series classification.
\end{abstract}

%
%

\ccsdesc[500]{Mathematics of computing~Time series analysis}
\ccsdesc[300]{Computing methodologies~Classification and regression trees}
\ccsdesc[100]{Information systems~Clustering}

%
\keywords{time series, distance metric, robustness, classification, clustering}

%
\maketitle

\section{Introduction}\label{sec:introduction}
Time series data classification is an important task in many domains such as data mining, machine learning and econometrics.
Extensive past evaluations~\cite{ding2008querying} have shown that k-Nearest Neighbor (k-NN) classification is among the most competitive classification approaches for time series data.
In simple terms, k-NN classification assigns a query time series instance the class based on its k nearest neighbors in a labeled training set.
As such, the k-NN classifier is a distance-based classifier, since its distance function is the only component that discriminates between classes.
The same applies to distance-based clustering algorithms~\cite{keogh2003need,rani2012recent}.

The data mining and machine learning communities have proposed numerous different distance functions for improving classification and clustering accuracies on benchmark datasets~\cite{ding2008querying,abanda2019review} and for accelerating the practical computation~\cite{keogh1997probabilistic,keogh2000scaling,keogh2001dimensionality,xi2006fast,mueen2016extracting}.
However, there is an important aspect of distance measures that was not considered so far to the best of our knowledge:
\textit{Robustness against arbitrary data contamination}.
While previous research has proposed distance measures that are ``robust'' against additive white Gaussian noise~\cite{keogh1997fast} or against temporal misalignment~\cite{sakoe1990dynamic}, we follow the definition used in the field of robust statistics.
A crucial measure for determining the robustness of a distance function is breakdown point (BP) analysis~\cite{huber2011robust}.
The  asymptotic BP describes the amount of contamination in the data that an estimator (in this case a distance function) can tolerate before it will be fully biased by the contamination in the worst case.
For example, the Euclidean distance has an  asymptotic BP of zero: If a single observation in one of the time series instances it compares is contaminated to (plus or minus) infinity, then the Euclidean distance becomes infinite as well, regardless of the remaining observations.

To address this issue, one may propose to use the raw Edit distance~\cite{kukich1992techniques}, as it is robust against arbitrary contamination at a few observations.
However, Edit distance is susceptible to a different type of contamination that is routinely overlooked, as it is trivially fulfilled by most distance measures. 
If time series data are subject to a tiny contamination at every single data point, then the Edit distance will become very large.
One may be tempted to address this issue by defining a small tolerance interval suggested by Chen et al.~\cite{chen2005robust}, yet this is difficult if the variance of the data is large or time-dependent, which is a well-known behavior of many econometric time series~\cite{mandelbrot1967variation,engle1982autoregressive}.
Further, the time series classification accuracy of the raw Edit distance is poor for real time series data.
If one extends Edit distance to an elastic (non-lockstep) variant thereof, such as Edit distance with real penalty~\cite{chen2004marriage}, then the classification accuracy may increase, yet the asymptotic BP immediately drops to $0$ .

Since all existing distance measures either have a low asymptotic BP or else yield a low classification accuracy, we aim to fill this gap.
In this work, we propose a novel distance metric which is formally robust according to Huber's definition~\cite{huber2011robust} against a small percentage of contaminated observations, and robust against tiny deviations at many observations.
Additionally, we show that its classification accuracy is not significantly different from other distance metrics and that our metric has a worst-case computational complexity of $\mathcal{O}(n\log n)$.
The source code of our implementation and a script that reproduces all results can be found online.\footnote{\url{https://github.com/mtoller/robust-distance-metric}}

\section{Notation and Problem Formulation}\label{sec:problem_definition}
\subsection{Theoretical Concepts}
Let $x=\{x_t,t\in1,\mydots,n\}$ and $y=\{y_t,t\in 1,\mydots,n\}$ be two time series instances and $d:\mathbb{R}^n\times\mathbb{R}^n\rightarrow\mathbb{R}$ a distance function for comparing them.
For an efficient distance-based classification, it is advantageous if $d(\cdot)$ is a \textit{metric}, as this allows a variety of run-time acceleration techniques~\cite{faloutsos1994fast,hjaltason2003index}.
To be a metric, $d(\cdot)$ has to fulfill the following properties for all $x,y\in\mathbb{R}^n$:
\begin{itemize}
	\item $d(x,y) \ge 0$ \hfill\textit{Non-negativity}
	\item $d(x,y)=0\Leftrightarrow x=y$ \hfill\textit{Identity of Indiscernibles}
	\item $d(x,y)=d(y,x)$ \hfill\textit{Symmetry}
	\item $d(x,z)\le d(x,y)+d(y,z)$ \hfill\textit{Triangle Inequality}
\end{itemize}
A typical example for a distance metric is the Euclidean distance 
\begin{equation}
	e(x,y)=\sqrt{\sum_{t=1}^{n}(x_t-y_t)^2}.\label{eq:euclid}
\end{equation}
If a distance function $d(\cdot)$ fulfills all properties except the identity of indiscernibles, then it is called a \textit{pseudometric}.

To evaluate the robustness of a distance function, we adapt the definition of the breakdown point given in~\cite{huber2011robust}.
Specifically, let $d_{\sup}=\sup_{x,y\in\mathbb{R}^n}d(x,y)$ be the largest possible value the distance function can obtain theoretically.
Then, the breakdown point $\beta^\star_d(n)$ is given by
\begin{equation}
 \beta^\star_d(n) = \min\left\{\frac{k}{n}\Big\vert\sup d(x,x+K) = d_{\sup} \right\}
\end{equation}
where the supremum is over all $x\in\mathbb{R}^n$ and over all contamination processes $K=\{K_t,t\in1,\mydots,n\}$ that assume arbitrary non-zero values in at most $k$ positions and zero otherwise.
In simple terms, the breakdown point describes the highest percentage of contaminated observations that function $d(\cdot)$ can tolerate.
For example, it is evident that for the Euclidean distance contaminating a single time point suffices, i.e., if $K_1=\infty$ and $K_t=0$ for $t=2,\mydots,n$, then $e(x,x+K)=d_{\sup}=\infty$ for every $x\in\mathbb{R}^n$; thus, $\beta^\star_e(n)=1/n$.
For clarity, the \textit{asymptotic} BP is obtained by evaluating $\beta^\star_d(n)$ as $n$ tends to infinity.


\subsection{Classification-Specific Aspects}
To link the theoretical concept of breakdown points with practical classification, we formulate two classification-specific notions of robustness.
To this end, let $\mathcal{C} = \{\mathcal{C}_1,\mydots,\mathcal{C}_r$\} denote a set of time series classes and $d(\cdot)$ a candidate distance function.


\begin{definition}~\label{def:contamination_tolerance}
	\textit{Contamination Tolerance}: A distance function $d(\cdot)$ tolerates $\hat{k}$ contaminated observation w.r.t.\ $\mathcal{C}$ if 
	\begin{equation}
		\forall i,j\neq i \in 1,\mydots,r:\forall x\in\mathcal{C}_i:\; \forall y \in\mathcal{C}_{j}: d(x,x+K) < d(x,y).
		\label{eq:contamination_tolerance}
	\end{equation}
	holds for every contamination processes $K=\{K_t,t\in1,\mydots,n\}$ that assumes arbitrary non-zero values in at most $\hat{k}$ positions.
\end{definition}
Intuitively, assume that a distance function $d(\cdot)$ ideally separates class $\mathcal{C}_i$ from other classes $\mathcal{C}_{j,j\neq i}$.
Function $d(\cdot)$ will tolerate up to $\hat{k}$ contaminated observations if the distance between an uncontaminated time series instance $x$ and a contaminant variant thereof $x+K$ is smaller than the distance between $x$ and an instance from a different class $y$.

To specify imprecision invariance as mentioned in Section~\ref{sec:introduction}, i.e. invariance to tiny changes, we introduce an imprecision process $\{\varepsilon_t\}$ that is negligibly small at all $t$.
Specifically, we assume that, for all $t$, $|\varepsilon_t|\le\varepsilon_{\max}$, where $\varepsilon_{\max}$ is much smaller than the standard deviation (or some norm) of the time series $x$.

\begin{definition}
	\textit{Imprecision Invariance}: A distance function $d(\cdot)$ is invariant to an imprecision of $\varepsilon_{\max}$ w.r.t.\ $\mathcal{C}$ if
	\begin{equation}
		\forall i,j\neq i \in 1,\mydots,r:\forall x\in\mathcal{C}_i:\; \forall y \in\mathcal{C}_{j}: 
		d(x,x+\varepsilon) < d(x,y)
		\label{eq:imprecision invariance}
	\end{equation}
	holds for every imprecision processes $\varepsilon=\{\varepsilon_t,t\in1,\mydots,n\}$ that satisfies $|\varepsilon_t|\le\varepsilon_{\max}$ for every $t$.
\end{definition}
In other words, assume distance function $d(\cdot)$ perfectly discriminates class $\mathcal{C}_i$ from $\mathcal{C}_{j,j \neq i}$.
Function $d(\cdot)$ is invariant to an imprecision of $\varepsilon_{\max}$ if the distance between an instance $x$ and almost the same instance $x+\varepsilon$ is smaller than the distance between $x$ and an instance from another class $y$.

Contamination tolerance and imprecision invariance are very different properties.
There are not many metrics that fulfill both simultaneously: For example, no metric induced by an $\mathcal{L}^p$ norm with a finite $p \ge 1$ is contamination tolerant for any non-zero $\hat{k}$. Also, while many popular metrics are imprecision invariant, some metrics such as the Edit distance are susceptible to it.


\section{Methods}~\label{sec:ensemble}
In this section, we present a novel metric which can tolerate considerable contamination and is invariant to imprecision.
The metric is obtained by aggregating an ensemble of metrics and pseudometrics in a way that preserves their discriminatory power while guaranteeing robust results.

\subsection{Metric Ensemble Members}
The ensemble consists of three metrics and three pseudometrics.
The distances measured by these metrics are combined via a scaling function and an arbitrary $\mathcal{L}^p$ norm, with $p \ge 1$, to obtain the metric $\mathcal{E}(x,y)$.
A summary of the ensemble members can be seen in Table~\ref{tab:ensemble_members}.
\begin{table}
	\caption{The components of the proposed ensemble metric $\mathcal{E}$. The top three members are metrics, while the bottom three are pseudometrics.}
	\begin{tabular}{l|c}
		Member name & Definition \\\toprule
		Euclidean distance & $e(x,y)$\\
		Log-distance &$\ell(x,y)$\\
		Raw Edit distance & $\textrm{Edit}(x,y)$\\
		\midrule
		Robust Euclidean distance & $e(\ms(x),\ms(y))$\\
		Robust Log-distance &$\ell(\ms(x),\ms(y))$\\
		Robust Raw Edit distance & $\textrm{Edit}(\ms(x),\ms(y))$\\
		\bottomrule
	\end{tabular}
	
	\label{tab:ensemble_members}
\end{table}
\begin{definition}
	\textit{Log-distance}: Let $x,y\in\mathbb{R}^n$ be two real-valued $n$-dimensional observations. Then, the Log-distance $\ell(\cdot)$ between $x$ and $y$ is given by
	\begin{equation}
	\ell(x,y)=\sum_{t=1}^n\log(1+|x_t-y_t|).
	\end{equation}
\end{definition}

\begin{proposition}
	The Log-distance $\ell(\cdot)$ is a metric. \label{prop:log_distance}
\end{proposition}
\begin{proof}
	Since $\log(x): \mathbb{R}^+\rightarrow\mathbb{R}$ is a strictly monotonic subadditive function, $\log(1+x)$ is also a strictly monotonic subadditive function that is zero iff $x=0$.
	Consequently, $\log(1+|x-y|)$ also fulfills these properties and is a metric by Kelly's theorem~\cite[p.~131]{kelley2017general}. That the sum of metrics is a metric~\cite{dobovs1998metric} completes the proof.
\end{proof}

For 1-dimensional data, the Log-distance is asymptotically smaller than any $\mathcal{L}^p$ metric with $p \ge 1$, since the logarithm grows slower than an arbitrary polynomial, i.e. $\lim\limits_{z\to\infty}\frac{\log(z
)}{P(z)}=0,\quad z\in\mathbb{R}$.
This property is beneficial when one expects a small number of large outliers in time series data $x_t$.
$\mathcal{L}^p$ metrics such as the Euclidean distance will be much more influenced by a single large difference than several small deviations that sum up to the same value.
The Log-distance will weight several small changes higher than one large change due to subadditivity of the logarithm.

The remaining two metrics of the ensemble are the Euclidean distance $e(\cdot)$ 
as defined in Equation~\eqref{eq:euclid} and the raw Edit distance
\begin{equation}
	\textrm{Edit}(x,y) = \sum_{t=1}^{n} \phi_t,\quad \phi_t=\begin{cases}
	0 & x_t=y_t\\
	1 & x_t\neq y_t\\
	\end{cases}
\end{equation}
which is equivalent to the number of observations where $x_t$ and $y_t$ differ.
While the Edit distance tolerates up to $n-1$ contaminated observations and is sensitive to imprecision, the Euclidean distance $e(\cdot)$ is invariant to imprecision but sensitive to contamination.
The Log-distance $\ell(\cdot)$ aims to present a middle-ground between the two.
Compared to the Euclidean distance it is ``more'' sensitive  to imprecision and ``less'' sensitive to contamination, while the inverse holds when it is compared against the Edit distance.
However, in terms of robustness, the Euclidean distance and the Log-distance are asymptotically equivalent, since they have the same BP $\beta_e^\star(n)=\beta_\ell^\star(n)=\frac{1}{n}$.
Hence, when confronted with arbitrary contamination, both metrics become equally useless in the worst case.

\subsection{Pseudometric Ensemble Members}
To raise the BP of the metrics in the ensemble $\mathcal{E}$, one can introduce a function composition with a function that has a high BP while preserving metric properties.
Let $m(x)$ be the median of $x$.
As a measure of central tendency, the median has a BP of $\beta_m^\star(n)=0.5+\frac{1}{n}$ according to Huber's definition~\cite{huber2011robust}.
However, computing the median of time series data $x_t$ is meaningless, since it disregards the temporal structure of $x_t$ by treating it like an unordered data set.
To exploit the asymptotic robustness of the median in the context of time series, one can instead apply the median via a sliding window:

\begin{definition}
	Let $x$ be a time series instance and let $w$, an odd integer in $[3;n]$, be the size of a sliding window.
	The sliding median $\overset{\rightarrow}{m}: \mathbb{R}^n~\rightarrow~\mathbb{R}^{n-w+1}$ of $x$ is then defined as
	\begin{equation}
		\ms(x) = \{m(x_1,\mydots,x_w),m(x_2,\mydots,x_{w+1}),\mydots,m(x_{n-w+1},\mydots,x_{n})\}.
	\end{equation}
\end{definition}

If one computes the Euclidean, Log and Edit distance of $\ms$, then the result is no longer a metric --- the identity of indiscernibles becomes violated since the median is not an injective function.
However, the remaining metric properties are preserved:
\begin{proposition}\label{prop:median_pseudometric}
	Let $d(\cdot)$ be a metric. Then the sliding median distance $M_d(x,y)=d(\ms(x),\ms(y))$ is a pseudometric.
\end{proposition}
\begin{proof}
	Non-negativity, symmetry and triangle inequality follow trivially from the application of $d(\cdot)$.
\end{proof}

The sliding median distance $M_d(\cdot)$ has a BP of $\frac{w}{2n}$, since, if all contamination occurred at $\frac{w+1}{2}$ consecutive observations, the median of all $\frac{w-1}{2}$ windows containing these observations could be contaminated to an arbitrary value.

\subsection{Combining the Members}
Since all six ensemble members operate on different scales, it is desirable to convert them to the same scale without loss of generality.
Therefore we propose the following scaling function $S(\cdot)$ that is applied after distance computation and that preserves all metric properties:

\begin{definition}
	Let $d(\cdot)$ be an arbitrary distance function. The metric-preserving scaling $S:\mathbb{R^+}\rightarrow[0;1]$ of this metric is then defined as
	\begin{equation}
		S(d(x,y))=1-\frac{1}{1+d(x,y)}.
	\end{equation}
\end{definition}
\begin{lemma}
	Metrics are closed under scaling with $S(\cdot)$. \label{lemma:scaling}
\end{lemma}
\begin{proof}
	$S(\cdot)$ is a concave, monotonically increasing function with $S(d(x,y))=0\Leftrightarrow d(x,y)=0$.
	Hence, it is a metric by Kelly's theorem~\cite{kelley2017general}.
\end{proof}

After scaling, the ensemble members can be combined into a single metric $\mathcal{E}(\cdot)$ via an arbitrary $\mathcal{L}^p$ norm with $p \ge 1$.
Specifically, we suggest the $\mathcal{L}^2$ norm.
The resulting function is a metric, since the sum of a pseudometric and a metric is a metric.
In particular, we propose the following ensemble:

\begin{equation}
	\mathcal{E}(x,y) \coloneqq \sqrt{
		\begin{aligned}
		&S(e(x,y))^2+S(\ell(x,y))^2+S(\textrm{Edit}(x,y))^2\\
		&+S(e(\ms(x)),\ms(y))^2+S(\ell(\ms(x),\ms(y))^2\\ &+S(\mathrm{Edit}(\ms(x),\ms(y)))^2
		\end{aligned}
	}		
\end{equation}

The ensemble $\mathcal{E}:\mathbb{R}^n\rightarrow[0;\sqrt{6}]$ has a BP of $\beta_\mathcal{E}^\star=1$.
This follows from the fact that the measurements of the non-robust metrics $e(\cdot)$ and $\ell(\cdot)$ are mapped onto the interval $[0,\sqrt{2}]$ and thus their influence on the ensemble is restricted.
The remaining members have a BP of $\frac{w}{2n}$ or higher, and the inclusion of the Edit distance raises the total BP to $1$.

The ensemble has a worst-case computational complexity of $\mathcal{O}(n\log n)$ under the assumption that $w=\mathcal{O}(n)$.
This arises from the ensemble's most expensive step, which is the computation of the sliding median $\ms$.
A more detailed explanation can be found in the Appendix.
%
%

\section{Practical Evaluation}\label{sec:evaluation}
In this section, we describe the experiments we conducted to show that the proposed ensemble metric $\mathcal{E}$ has competitive classification accuracy.
Further, we validate that $\mathcal{E}$ tolerates contamination and is imprecision invariant.
We compared $\mathcal{E}$ with Euclidean distance (Euc), Dynamic Time Warping~\cite{sakoe1990dynamic} (DTW) with window size $w=100$, Log-distance (Log), raw Edit distance (ED), and edit distance with a tolerance interval (EDR) set to $10\%$ of the median absolute deviation.

\subsection{Setup}
In our practical evaluation, we conducted three experiments to assess the following properties of the ensemble $\mathcal{E}$:
\begin{itemize}
	\item 1-NN classification error rate
	\item Contamination tolerance (cf. Equation~\eqref{eq:contamination_tolerance})
	\item Imprecision invariance (cf. Equation~\eqref{eq:imprecision invariance}) 
\end{itemize}
For all three experiments we used 83 selected benchmark data from the UCR Time Series Classification Archive~\cite{UCRArchive2018}.
All datasets which contained non-real data such as missing values were omitted, which was necessary since otherwise the behavior of the ensemble would be undefined.
Further, we were forced to omit all datasets in which either training or test datasets contained more that 1000 instances due to our limited computational resources.

For the classification accuracy experiment we computed the raw accuracy of a 1-NN classifier based on the ensemble metric and subtracted the resulting value from 1 to obtain the error rate.
To determine statistical significance, we then performed a Friedman's rank test~\cite{demvsar2006statistical}.

The dataset-dependent contamination tolerance was computed with the following procedure:
\begin{enumerate}[label=\roman*)]
	\item Assume $d(\cdot)$ perfectly separates all classes in the dataset.
	\item For every instance $x\in\mathcal{C}_i$, count for how many $y\in\mathcal{C}_{j,j\neq i}$ Equation~\eqref{eq:contamination_tolerance} holds when $k=0.05\times n$, i.e. 5\% of the observations are contaminated to be $\pm\infty$
	\item Compute the ratio of this count and the number of instances in $\mathcal{C}_{j\neq i}$.
	\item Compute the mean over all instance-based ratios.
\end{enumerate}
The window size of the ensemble $\mathcal{E}$ was set to $w=0.1\times n+1$ observations, which ensures that the window is always large enough to be resilient against 5\% contamination.
The imprecision invariance was computed similarly, only with Equation~\eqref{eq:contamination_tolerance} replaced by Equation~\eqref{eq:imprecision invariance}, $k$ set to $n$ and $\varepsilon_t\sim\mathcal{U}(-10^{-10},10^{-10})$.
Theoretically, $\varepsilon_t$ should be as close to zero as possible.
Yet the results below suggest that the above interval is sufficiently small for asserting the imprecision invariance property of the distance functions under consideration.

\subsection{Results}
In this subsection we present a summary of the results of the three experiments we conducted.
The complete results can be found in the Appendix.

Our first experiment showed that, in terms of classification error rate, there is no significant difference between the ensemble $\mathcal{E}$ and Euc, DTW or Log, but that ED and EDR are significantly worse than $\mathcal{E}$.
A visual representation of this result is depicted in Figure~\ref{fig:friedman}.

\begin{figure}[!ht]
	\includegraphics[width=\linewidth,trim={50 50 50 45}, clip]{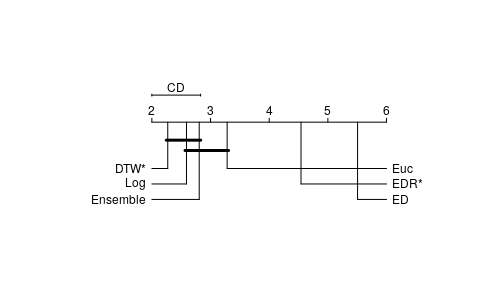}
	\caption{Critical distance plot for the classification error rate. Average ranks are depicted in order, where lower is better and the horizontal bars highlight no significant difference. The ensemble $\mathcal{E}$ is not significantly different from either Euc, DTW or Log, but significantly better than ED and EDR. Functions labeled with an asterisk (*) are not metrics.}
	\label{fig:friedman}
\end{figure}

The second and third experiment revealed that the only distance functions which are both contamination tolerant and imprecision invariant are the ensemble $\mathcal{E}$ and EDR.
An overview of these results can be found in Table~\ref{tab:short_results}.


\begin{table}[!ht]
	\setlength{\tabcolsep}{4pt}
	\caption{Summary of the second and third experiments. The ensemble $\mathcal{E}$ and EDR tolerate contamination on 56 and 79 data sets, respectively, while the other distance measures never tolerate contamination. In terms of imprecision invariance, the first four distances are perfectly invariant, while ED is susceptible and EDR is almost perfectly invariant.}
	\begin{tabular}{l|cccccc}
		&$\mathcal{E}$&Euc&DTW&Log&ED&EDR\\\toprule
		Is a metric? &\checkmark &\checkmark&$\times$&\checkmark&\checkmark&$\times$\\
		\midrule
		Contam. tol. on \# datasets&56&0&0&0&83&81\\
		Imprec. Invar. on \# datasets&83&83&83&83&0&79	\\
		Both on \# datasets&56&0&0&0&0&79\\
		
	\end{tabular}

\label{tab:short_results}
\end{table}

\section{Discussion and Conclusion}
The goal of this work was to propose a distance function that \begin{itemize}
	\item is robust against arbitrary contamination
	\item is invariant to imprecision
	\item fulfills all metric properties
	\item has a competitive classification accuracy
	\item is computationally efficient.
\end{itemize}

The combined results of our theoretical analysis and of the practical evaluation suggest that the ensemble $\mathcal{E}$ has all of these properties.
One might argue that the ensemble $\mathcal{E}$ is no improvement over EDR, since both methods depend on one parameter which influences their classification accuracy.
However, correctly choosing an appropriate tolerance interval for time series with large or time-dependent variance is difficult, while tuning the ensemble's window size is simpler --- it should be just larger than the expected amount of contamination.
Additionally, $\mathcal{E}$ has a significantly better classification accuracy than EDR.
The same holds for the Log-distance, which we believe should be seen as a natural alternative to the Euclidean distance.



Future work might consider less extreme cases of contamination and determine precisely how this affects classification accuracy.
Further, evaluating robust distance functions on a clustering task with arbitrarily contaminated data seems a promising avenue for the future.

\section{Acknowledgments}
Our work was funded by the iDev40 project.
The iDev40 project  has  received  funding  from  the  ECSEL  Joint  Undertaking (JU) under  grant  agreement  No  783163. The JU receives support from the European Union's Horizon  2020  research  and  innovation  programme. It  is  co-funded  by  the  consortium  members, grants from Austria, Germany, Belgium, Italy, Spain and Romania.

%
\bibliographystyle{ACM-Reference-Format}
\bibliography{background}

%
\appendix
\section{Detailed Time Complexity Analysis}

A procedural description of the ensemble $\mathcal{E}$ is listed in Algorithm~\ref{alg:ensemble}.
Since there are no directly listed loops, the ensemble is linear in the complexity of the functions it applies.
Hence, its worst-case time complexity is equal to that of the function that requires the most expensive computation.
When looking at the distance functions, it quickly becomes evident that these can be computed in $\mathcal{O}(n)$, since these lockstep methods look at each observation only once.
The scaling function can be computed in $\mathcal{O}(1)$, so the only non-trivial step is the computation of the sliding median.

\begin{algorithm}[!ht]
	\caption{Ensemble Metric $\mathcal{E}$}
	\begin{algorithmic}
		\REQUIRE $x_t,y_t,w$
		\IF {$w\cong 0 \;(\textrm{mod}\: 2)$}
		\STATE $w\leftarrow w+1$
		\ENDIF
		\STATE $\text{med}_x \leftarrow \ms(x)$
		\STATE $\text{med}_y \leftarrow \ms(y)$
		\STATE $\text{dist}_1 \leftarrow S_e(x,y)$
		\STATE $\text{dist}_2 \leftarrow S_\ell(x,y)$
		\STATE $\text{dist}_3 \leftarrow S_\textrm{Edit}(x,y)$
		\STATE $\text{dist}_4 \leftarrow S_e(\text{med}_x,\text{med}_y)$
		\STATE $\text{dist}_5 \leftarrow S_\ell(\text{med}_x,\text{med}_y)$
		\STATE $\text{dist}_6 \leftarrow S_\textrm{Edit}(\text{med}_x,\text{med}_y)$
		\RETURN $\sqrt{\sum_{i}^{6}\text{dist}_i^2}$

	\end{algorithmic}
	\label{alg:ensemble}
\end{algorithm}
Directly computing the sliding median requires one to sort the observations in all windows.
Since the most efficient sorting algorithm requires $\mathcal{O}(w\log w)$ steps, sorting all $n-w+1$ windows would take $(n-w+1)\times\mathcal{O}(w\log w)=n\times\mathcal{O}(w\log w)$ steps.
We assume that a certain small percentage of the $n$ observations is contaminated, so we must conclude that $w=\mathcal{O}(n)$, which results in a total complexity of $\mathcal{O}(n^2\log n)$.

However, there exist more efficient algorithms for computing the sliding median, and several implementations are available in C libraries.
If one starts by sorting the initial window, one consumes $\mathcal{O}(n\log n)$ time as argued above.
After this, if one keeps the computed window and an index list in the memory, the next window can be computed by removing the oldest observation in $\mathcal{O}(1)$ time and sorting in the new observation in $\mathcal{O}(\log n)$ time.
So, for all windows, this algorithms requires $\mathcal{O}(n\log n) + (n-w)\times\mathcal{O}(\log n)=\mathcal{O}(n \log n)$ steps.
\section{Proofs}
This section contains alternate proofs of the propositions and lemmas presented in the main article.
These proofs do no rely on Kelly's theorem and are easy to verify.
\subsection{Proof of Proposition~\ref{prop:log_distance}}

\begin{proof}
	\textit{Non-negativity}: $\ell(x,y)\ge 0$
	\[\ell(x,y)=\log(1+|x-y|)\ge \log(1)=0\ge 0 \]
	
	\textit{Identity of Indiscernibles}: $\ell(x,y)=0\Leftrightarrow x=y$
	
	This is trivial, since the $\log$ function has exactly one zero at $\log(1+|x-x|)=log(1)=0$
	\vspace{0.5cm}
	
	\textit{Symmetry}
	
	Trivial, due to absolute value.
	\vspace{0.5cm}
	
	\textit{Triangle Inequality}: $\ell(x,z)\le \ell(x,y)+\ell(y,z)$
	\begin{align*}
	\ell(x,z) = \log(1+|x-z|) &\le\log(1+|x-y|)+\log(1+|y-z|)\\
	\log(1+|x-z|)&\le\log\Bigl((1+|x-y|)(1+|y-z|)\Bigr)\\
	1+|x-z| &\le (1+|x-y|)(1+|y-z|)\\
	|x-z| &\le |x-y|+|y-z|+|x-y||y-z|\\
	|x-z| &\le |x-y|+|y-z|
	\end{align*}
\end{proof}

\subsection{Proof of Proposition~\ref{prop:median_pseudometric}}
\begin{proof}
	\textit{Non-negativity}: $M_d(x,y) \ge 0$
	
	This follows trivially from the fact that $d(\cdot) \ge 0$.
	\vspace{0.5cm}
	
	\textit{Symmetry}: $M_d(x,y)=M_d(y,x)$
	
	This also follows trivially from the symmetry of $d(\cdot)$.
	\vspace{0.5cm}
	
	\textit{Triangle Inequality}: $M_d(x,z) \le M_d(x,y)+M_d(y,z)$
		
	\[d(\ms(x),\ms(z)) \le d(\ms(x),\ms(y)+d(\ms(y),\ms(z)\]
	\[a\coloneqq \ms(x),\quad b\coloneqq \ms(y),\quad c\coloneqq \ms(z)\]
	\[d(a,c)\le d(a,b)+d(b,c) \]
\end{proof}

\subsection{Proof of Lemma~\ref{lemma:scaling}}

\begin{proof}
	
	\textit{Non-negativity}: $S(x,y) \ge 0$
	\begin{align*}
		1-\frac{1}{1+d(x,y)} &\ge 0\\
		\frac{1}{1+d(x,y)}&\le 1\\
		1+d(x,y)&\ge 1\\
		d(x,y)&\ge 0
	\end{align*}
	
	\textit{Identity of Indiscernibles}: $S(x,y)=0 \Leftrightarrow x=y$
	
	First we show $S(x,y) = 0 \implies x=y$
	\begin{align*}
		1-\frac{1}{1+d(x,y)} &= 0\\
		1+d(x,y) &=1\\
		d(x,y)&=0\\
	\end{align*}
	
	Now we show $S(x,y) = 0 \impliedby x=y$
	\begin{align*}
		S(x,x)&=1-\frac{1}{1+d(x,x)}\\
		&=1-\frac{1}{1}\\
		&=0
	\end{align*}
	
	\textit{Symmetry}: $S(x,y)=S(y,x)$
	
	This is trivial, since $d(x,y)$ is symmetric.
	\vspace{0.5cm}
	
	\textit{Triangle Inequality}: $S(x,z)\le S(x,y)+S(y,z)$
	
	\begin{align*}
		&1-\frac{1}{1+d(x,z)} \le 1-\frac{1}{1+d(x,y)} + 1-\frac{1}{1+d(y,z)}\\
		&-\frac{1}{1+d(x,z)} \le 1-\frac{1}{1+d(x,y)}-\frac{1}{1+d(y,z)}\\
		&-1 \le (1+d(x,z))-\frac{1+d(x,z)}{1+d(x,y)}-\frac{1+d(x,z)}{1+d(y,z)}\\
		&-\bigl(1+d(x,y)\bigr)\bigl(1+d(y,z)\bigr) \le \bigl(1+d(x,z)\bigr)\Bigl(\bigl(1+d(x,y)\bigr)\bigl(1+d(y,z)\bigr)\\  &  \hspace{4cm}-\bigl(1+d(x,y)\bigr)-\bigl(1+d(y,z)\bigr)\Bigr)\\
		&-\bigl(1+d(x,y)\bigr)\bigl(1+d(y,z)\bigr) \le\bigl(1+d(x,z)\bigr) \bigl(d(x,y)d(y,z)-1\bigr)\\
		&-1-d(x,y)d(y,z)-d(x,y)-d(y,z)\le d(x,y)d(y,z)-1-d(x,z)\\
		&\hspace{5	cm}+d(x,y)d(y,z)d(x,z)\\
		&d(x,z)-d(x,y)-d(y,z) \le 0 \le d(x,y)d(y,z)d(x,z)\\
	\end{align*}
	
\end{proof}

\section{Full Empirical Results}
This section contains tables with the complete empirical results.
The classification accuracy of Dynamic Time Warping was taken from the results published in the UCR archive~\cite{UCRArchive2018}.
The names of the used datasets and the complete results can be found below.
The first two tables list the classification error rate per distance function.
The second two table compare the contamination tolerance and the imprecision invariance per distance function.

\begin{table*}[!ht]
	\begin{tabular}{l|cccccc}
		Dataset&\multicolumn{6}{c}{\textit{Error rate}}\\
		&$\mathcal{E}$&Euc&DTW&Log&ED&EDR\\
		\toprule
		ACSF1 & 0.28 & 0.46 & 0.36 & 0.17 & 0.90 & 0.48 \\ 
		Adiac & 0.42 & 0.39 & 0.40 & 0.40 & 0.97 & 0.54 \\ 
		ArrowHead & 0.21 & 0.20 & 0.30 & 0.20 & 0.61 & 0.31 \\ 
		Beef & 0.33 & 0.33 & 0.37 & 0.40 & 0.80 & 0.50 \\ 
		BeetleFly & 0.25 & 0.25 & 0.30 & 0.35 & 0.50 & 0.40 \\ 
		BirdChicken & 0.45 & 0.45 & 0.25 & 0.35 & 0.50 & 0.35 \\ 
		BME & 0.25 & 0.17 & 0.10 & 0.20 & 0.65 & 0.59 \\ 
		Car & 0.27 & 0.27 & 0.27 & 0.28 & 0.77 & 0.30 \\ 
		CBF & 0.06 & 0.15 & 0.00 & 0.11 & 0.67 & 0.38 \\ 
		Chinatown & 0.06 & 0.05 & 0.04 & 0.05 & 0.45 & 0.03 \\ 
		Coffee & 0.04 & 0.00 & 0.00 & 0.07 & 0.46 & 0.11 \\ 
		Computers & 0.47 & 0.42 & 0.30 & 0.42 & 0.50 & 0.42 \\ 
		CricketX & 0.37 & 0.42 & 0.25 & 0.37 & 0.92 & 0.71 \\ 
		CricketY & 0.38 & 0.43 & 0.26 & 0.34 & 0.92 & 0.70 \\ 
		CricketZ & 0.38 & 0.41 & 0.25 & 0.36 & 0.93 & 0.72 \\ 
		DiatomSizeReduction & 0.07 & 0.07 & 0.03 & 0.08 & 0.70 & 0.08 \\ 
		DistalPhalanxOutlineAgeGroup & 0.35 & 0.37 & 0.23 & 0.33 & 0.58 & 0.26 \\ 
		DistalPhalanxOutlineCorrect & 0.29 & 0.28 & 0.28 & 0.26 & 0.42 & 0.30 \\ 
		DistalPhalanxTW & 0.37 & 0.37 & 0.41 & 0.37 & 0.70 & 0.34 \\ 
		Earthquakes & 0.35 & 0.29 & 0.28 & 0.33 & 0.75 & 0.75 \\ 
		ECG200 & 0.12 & 0.12 & 0.23 & 0.11 & 0.64 & 0.20 \\ 
		ECGFiveDays & 0.18 & 0.20 & 0.23 & 0.21 & 0.50 & 0.36 \\ 
		EOGHorizontalSignal & 0.60 & 0.58 & 0.50 & 0.67 & 0.83 & 0.83 \\ 
		EOGVerticalSignal & 0.68 & 0.56 & 0.55 & 0.73 & 0.86 & 0.86 \\ 
		EthanolLevel & 0.72 & 0.73 & 0.72 & 0.69 & 0.75 & 0.69 \\ 
		FaceFour & 0.22 & 0.22 & 0.17 & 0.15 & 0.70 & 0.26 \\ 
		FiftyWords & 0.34 & 0.37 & 0.31 & 0.31 & 0.97 & 0.61 \\ 
		Fish & 0.22 & 0.22 & 0.18 & 0.23 & 0.88 & 0.30 \\ 
		Fungi & 0.12 & 0.18 & 0.16 & 0.08 & 0.96 & 0.51 \\ 
		GunPoint & 0.07 & 0.09 & 0.09 & 0.05 & 0.51 & 0.24 \\ 
		GunPointAgeSpan & 0.03 & 0.10 & 0.08 & 0.00 & 0.49 & 0.17 \\ 
		GunPointMaleVersusFemale & 0.01 & 0.03 & 0.00 & 0.01 & 0.47 & 0.21 \\ 
		GunPointOldVersusYoung & 0.00 & 0.05 & 0.16 & 0.00 & 0.52 & 0.01 \\ 
		Ham & 0.42 & 0.40 & 0.53 & 0.50 & 0.51 & 0.43 \\ 
		HandOutlines & 0.14 & 0.14 & 0.12 & 0.14 & 0.46 & 0.18 \\ 
		Haptics & 0.63 & 0.63 & 0.62 & 0.64 & 0.79 & 0.64 \\ 
		Herring & 0.47 & 0.48 & 0.47 & 0.41 & 0.41 & 0.55 \\ 
		HouseTwenty & 0.14 & 0.34 & 0.08 & 0.18 & 0.32 & 0.37 \\ 
		InlineSkate & 0.66 & 0.66 & 0.62 & 0.64 & 0.83 & 0.69 \\ 
		InsectEPGRegularTrain & 0.00 & 0.32 & 0.13 & 0.00 & 0.00 & 0.00 \\ 
		InsectEPGSmallTrain & 0.00 & 0.34 & 0.27 & 0.00 & 0.00 & 0.00 \\ 
		LargeKitchenAppliances & 0.47 & 0.51 & 0.21 & 0.42 & 0.67 & 0.67 \\ 
		Lightning2 & 0.18 & 0.25 & 0.13 & 0.18 & 0.46 & 0.51 \\ 
		Lightning7 & 0.33 & 0.42 & 0.27 & 0.26 & 0.74 & 0.71 \\ 
		Meat & 0.07 & 0.07 & 0.07 & 0.07 & 0.65 & 0.07 \\

	\end{tabular}
	\caption{Classification error per distance function, Part 1. The ensemble $\mathcal{E}$ is not significantly different from DTW, Log or Euc. ED and EDR frequently have a higher classification error than the remaining functions. Unsurprisingly, DTW has the lowest overall error rate. Its elastic nature likely superior classification accuracy over lockstep distance functions.}
	\label{tab:results11}
\end{table*}
\begin{table*}
\begin{tabular}{l|cccccc}
	Dataset&\multicolumn{6}{c}{\textit{Error rate}}\\
	&$\mathcal{E}$&Euc&DTW&Log&ED&EDR\\
		\toprule
		MedicalImages & 0.31 & 0.32 & 0.26 & 0.29 & 0.49 & 0.53 \\ 
		MiddlePhalanxOutlineAgeGroup & 0.45 & 0.48 & 0.50 & 0.47 & 0.44 & 0.49 \\ 
		MiddlePhalanxOutlineCorrect & 0.25 & 0.23 & 0.30 & 0.25 & 0.43 & 0.24 \\ 
		MiddlePhalanxTW & 0.49 & 0.49 & 0.49 & 0.45 & 0.73 & 0.47 \\ 
		OliveOil & 0.13 & 0.13 & 0.17 & 0.17 & 0.83 & 0.60 \\ 
		OSULeaf & 0.48 & 0.48 & 0.41 & 0.45 & 0.90 & 0.55 \\ 
		PigAirwayPressure & 0.88 & 0.94 & 0.89 & 0.90 & 0.90 & 0.94 \\ 
		PigArtPressure & 0.73 & 0.88 & 0.75 & 0.72 & 0.93 & 0.86 \\ 
		PigCVP & 0.87 & 0.92 & 0.85 & 0.87 & 0.91 & 0.88 \\ 
		Plane & 0.04 & 0.04 & 0.00 & 0.04 & 0.85 & 0.01 \\ 
		PowerCons & 0.03 & 0.07 & 0.12 & 0.04 & 0.20 & 0.32 \\ 
		ProximalPhalanxOutlineAgeGroup & 0.23 & 0.21 & 0.20 & 0.22 & 0.57 & 0.22 \\ 
		ProximalPhalanxOutlineCorrect & 0.24 & 0.19 & 0.22 & 0.22 & 0.32 & 0.24 \\ 
		ProximalPhalanxTW & 0.28 & 0.29 & 0.24 & 0.30 & 0.98 & 0.27 \\ 
		RefrigerationDevices & 0.57 & 0.61 & 0.54 & 0.52 & 0.67 & 0.69 \\ 
		Rock & 0.44 & 0.16 & 0.40 & 0.34 & 0.62 & 0.46 \\ 
		ScreenType & 0.65 & 0.64 & 0.60 & 0.62 & 0.67 & 0.72 \\ 
		SemgHandGenderCh2 & 0.13 & 0.24 & 0.20 & 0.22 & 0.35 & 0.36 \\ 
		SemgHandMovementCh2 & 0.24 & 0.63 & 0.42 & 0.55 & 0.83 & 0.82 \\ 
		SemgHandSubjectCh2 & 0.16 & 0.60 & 0.27 & 0.42 & 0.80 & 0.77 \\ 
		ShapeletSim & 0.52 & 0.46 & 0.35 & 0.49 & 0.50 & 0.50 \\ 
		ShapesAll & 0.24 & 0.25 & 0.23 & 0.24 & 0.98 & 0.37 \\ 
		SmallKitchenAppliances & 0.46 & 0.66 & 0.36 & 0.51 & 0.66 & 0.74 \\ 
		SmoothSubspace & 0.02 & 0.09 & 0.17 & 0.00 & 0.67 & 0.15 \\ 
		SonyAIBORobotSurface1 & 0.24 & 0.30 & 0.27 & 0.31 & 0.57 & 0.37 \\ 
		SonyAIBORobotSurface2 & 0.13 & 0.14 & 0.17 & 0.12 & 0.38 & 0.21 \\ 
		Strawberry & 0.06 & 0.05 & 0.06 & 0.05 & 0.36 & 0.05 \\ 
		SwedishLeaf & 0.22 & 0.21 & 0.21 & 0.22 & 0.93 & 0.32 \\ 
		Symbols & 0.10 & 0.10 & 0.05 & 0.10 & 0.84 & 0.20 \\ 
		SyntheticControl & 0.07 & 0.12 & 0.01 & 0.13 & 0.83 & 0.35 \\ 
		ToeSegmentation1 & 0.29 & 0.32 & 0.23 & 0.27 & 0.46 & 0.36 \\ 
		ToeSegmentation2 & 0.15 & 0.19 & 0.16 & 0.12 & 0.18 & 0.21 \\ 
		Trace & 0.31 & 0.24 & 0.00 & 0.21 & 0.76 & 0.32 \\ 
		UMD & 0.25 & 0.24 & 0.01 & 0.24 & 0.48 & 0.51 \\ 
		Wine & 0.33 & 0.39 & 0.43 & 0.35 & 0.50 & 0.48 \\ 
		WordSynonyms & 0.38 & 0.38 & 0.35 & 0.34 & 0.91 & 0.63 \\ 
		Worms & 0.48 & 0.55 & 0.42 & 0.56 & 0.44 & 0.64 \\ 
		WormsTwoClass & 0.36 & 0.39 & 0.38 & 0.42 & 0.44 & 0.43 \\ 
	\end{tabular}
	\caption{Classification error per distance function, Part 2.}
	\label{tab:results12}
\end{table*}
\begin{table*}[!ht]
	\begin{tabular}{l|cccccc|cccccc}
		Dataset& \multicolumn{6}{c}{\textit{Contamination Tolerance}}&
		\multicolumn{6}{c}{\textit{Imprecision Invariance}}\\
		&$\mathcal{E}$&Euc&DTW&Log&ED&EDR&$\mathcal{E}$&Euc&DTW&Log&ED&EDR\\
		\toprule
		ACSF1 & 0.68 & 0.00 & 0.00 & 0.00 & 1.00 & 1.00 & 1.00 & 1.00 & 1.00 & 1.00 & 0.00 & 0.98 \\ 
		Adiac & 0.49 & 0.00 & 0.00 & 0.00 & 1.00 & 0.94 & 1.00 & 1.00 & 1.00 & 1.00 & 0.00 & 0.98 \\ 
		ArrowHead & 0.98 & 0.00 & 0.00 & 0.00 & 1.00 & 1.00 & 1.00 & 1.00 & 1.00 & 1.00 & 0.00 & 1.00 \\ 
		Beef & 0.85 & 0.00 & 0.00 & 0.00 & 1.00 & 1.00 & 1.00 & 1.00 & 1.00 & 1.00 & 0.00 & 1.00 \\ 
		BeetleFly & 1.00 & 0.00 & 0.00 & 0.00 & 1.00 & 1.00 & 1.00 & 1.00 & 1.00 & 1.00 & 0.00 & 1.00 \\ 
		BirdChicken & 1.00 & 0.00 & 0.00 & 0.00 & 1.00 & 1.00 & 1.00 & 1.00 & 1.00 & 1.00 & 0.00 & 1.00 \\ 
		BME & 1.00 & 0.00 & 0.00 & 0.00 & 1.00 & 1.00 & 1.00 & 1.00 & 1.00 & 1.00 & 0.00 & 1.00 \\ 
		Car & 0.99 & 0.00 & 0.00 & 0.00 & 1.00 & 1.00 & 1.00 & 1.00 & 1.00 & 1.00 & 0.00 & 1.00 \\ 
		CBF & 1.00 & 0.00 & 0.00 & 0.00 & 1.00 & 1.00 & 1.00 & 1.00 & 1.00 & 1.00 & 0.00 & 1.00 \\ 
		Chinatown & 1.00 & 0.00 & 0.00 & 0.00 & 1.00 & 1.00 & 1.00 & 1.00 & 1.00 & 1.00 & 0.00 & 1.00 \\ 
		Coffee & 0.30 & 0.00 & 0.00 & 0.00 & 1.00 & 1.00 & 1.00 & 1.00 & 1.00 & 1.00 & 0.00 & 1.00 \\ 
		Computers & 1.00 & 0.00 & 0.00 & 0.00 & 1.00 & 1.00 & 1.00 & 1.00 & 1.00 & 1.00 & 0.00 & 0.99 \\ 
		CricketX & 1.00 & 0.00 & 0.00 & 0.00 & 1.00 & 1.00 & 1.00 & 1.00 & 1.00 & 1.00 & 0.00 & 1.00 \\ 
		CricketY & 1.00 & 0.00 & 0.00 & 0.00 & 1.00 & 1.00 & 1.00 & 1.00 & 1.00 & 1.00 & 0.00 & 1.00 \\ 
		CricketZ & 1.00 & 0.00 & 0.00 & 0.00 & 1.00 & 1.00 & 1.00 & 1.00 & 1.00 & 1.00 & 0.00 & 1.00 \\ 
		DiatomSizeReduction & 0.95 & 0.00 & 0.00 & 0.00 & 1.00 & 1.00 & 1.00 & 1.00 & 1.00 & 1.00 & 0.00 & 1.00 \\ 
		DistalPhalanxOutlineAgeGroup & 0.95 & 0.00 & 0.00 & 0.00 & 1.00 & 1.00 & 1.00 & 1.00 & 1.00 & 1.00 & 0.00 & 1.00 \\ 
		DistalPhalanxOutlineCorrect & 0.66 & 0.00 & 0.00 & 0.00 & 1.00 & 1.00 & 1.00 & 1.00 & 1.00 & 1.00 & 0.00 & 1.00 \\ 
		DistalPhalanxTW & 0.94 & 0.00 & 0.00 & 0.00 & 1.00 & 1.00 & 1.00 & 1.00 & 1.00 & 1.00 & 0.00 & 1.00 \\ 
		Earthquakes & 1.00 & 0.00 & 0.00 & 0.00 & 1.00 & 1.00 & 1.00 & 1.00 & 1.00 & 1.00 & 0.00 & 1.00 \\ 
		ECG200 & 1.00 & 0.00 & 0.00 & 0.00 & 1.00 & 1.00 & 1.00 & 1.00 & 1.00 & 1.00 & 0.00 & 1.00 \\ 
		ECGFiveDays & 1.00 & 0.00 & 0.00 & 0.00 & 1.00 & 1.00 & 1.00 & 1.00 & 1.00 & 1.00 & 0.00 & 1.00 \\ 
		EOGHorizontalSignal & 1.00 & 0.00 & 0.00 & 0.00 & 1.00 & 1.00 & 1.00 & 1.00 & 1.00 & 1.00 & 0.16 & 1.00 \\ 
		EOGVerticalSignal & 1.00 & 0.00 & 0.00 & 0.00 & 1.00 & 1.00 & 1.00 & 1.00 & 1.00 & 1.00 & 0.15 & 1.00 \\ 
		EthanolLevel & 0.85 & 0.00 & 0.00 & 0.00 & 1.00 & 1.00 & 1.00 & 1.00 & 1.00 & 1.00 & 0.00 & 1.00 \\ 
		FaceFour & 1.00 & 0.00 & 0.00 & 0.00 & 1.00 & 1.00 & 1.00 & 1.00 & 1.00 & 1.00 & 0.00 & 1.00 \\ 
		FiftyWords & 1.00 & 0.00 & 0.00 & 0.00 & 1.00 & 1.00 & 1.00 & 1.00 & 1.00 & 1.00 & 0.00 & 1.00 \\ 
		Fish & 0.97 & 0.00 & 0.00 & 0.00 & 1.00 & 1.00 & 1.00 & 1.00 & 1.00 & 1.00 & 0.00 & 1.00 \\ 
		Fungi & 1.00 & 0.00 & 0.00 & 0.00 & 1.00 & 1.00 & 1.00 & 1.00 & 1.00 & 1.00 & 0.00 & 1.00 \\ 
		GunPoint & 1.00 & 0.00 & 0.00 & 0.00 & 1.00 & 1.00 & 1.00 & 1.00 & 1.00 & 1.00 & 0.00 & 1.00 \\ 
		GunPointAgeSpan & 1.00 & 0.00 & 0.00 & 0.00 & 1.00 & 1.00 & 1.00 & 1.00 & 1.00 & 1.00 & 0.05 & 1.00 \\ 
		GunPointMaleVersusFemale & 1.00 & 0.00 & 0.00 & 0.00 & 1.00 & 1.00 & 1.00 & 1.00 & 1.00 & 1.00 & 0.07 & 1.00 \\ 
		GunPointOldVersusYoung & 1.00 & 0.00 & 0.00 & 0.00 & 1.00 & 1.00 & 1.00 & 1.00 & 1.00 & 1.00 & 0.04 & 1.00 \\ 
		Ham & 1.00 & 0.00 & 0.00 & 0.00 & 1.00 & 1.00 & 1.00 & 1.00 & 1.00 & 1.00 & 0.00 & 1.00 \\ 
		HandOutlines & 0.57 & 0.00 & 0.00 & 0.00 & 1.00 & 1.00 & 1.00 & 1.00 & 1.00 & 1.00 & 0.00 & 1.00 \\ 
		Haptics & 1.00 & 0.00 & 0.00 & 0.00 & 1.00 & 1.00 & 1.00 & 1.00 & 1.00 & 1.00 & 0.00 & 1.00 \\ 
		Herring & 0.62 & 0.00 & 0.00 & 0.00 & 1.00 & 1.00 & 1.00 & 1.00 & 1.00 & 1.00 & 0.00 & 1.00 \\ 
		HouseTwenty & 1.00 & 0.00 & 0.00 & 0.00 & 1.00 & 1.00 & 1.00 & 1.00 & 1.00 & 1.00 & 0.18 & 1.00 \\ 
		InlineSkate & 1.00 & 0.00 & 0.00 & 0.00 & 1.00 & 1.00 & 1.00 & 1.00 & 1.00 & 1.00 & 0.00 & 1.00 \\ 
		InsectEPGRegularTrain & 1.00 & 0.00 & 0.00 & 0.00 & 1.00 & 1.00 & 1.00 & 1.00 & 1.00 & 1.00 & 0.00 & 1.00 \\ 
		InsectEPGSmallTrain & 1.00 & 0.00 & 0.00 & 0.00 & 1.00 & 1.00 & 1.00 & 1.00 & 1.00 & 1.00 & 0.00 & 1.00 \\ 
		LargeKitchenAppliances & 0.96 & 0.00 & 0.00 & 0.00 & 1.00 & 1.00 & 1.00 & 1.00 & 1.00 & 1.00 & 0.00 & 1.00 \\ 
		Lightning2 & 1.00 & 0.00 & 0.00 & 0.00 & 1.00 & 1.00 & 1.00 & 1.00 & 1.00 & 1.00 & 0.00 & 1.00 \\ 
		Lightning7 & 1.00 & 0.00 & 0.00 & 0.00 & 1.00 & 1.00 & 1.00 & 1.00 & 1.00 & 1.00 & 0.00 & 1.00 \\ 
		Meat & 0.00 & 0.00 & 0.00 & 0.00 & 1.00 & 1.00 & 1.00 & 1.00 & 1.00 & 1.00 & 0.00 & 1.00 \\

	\end{tabular}
	\caption{Contamination tolerance and imprecision invariance per distance function, Part 1.
	 The numbers indicate the percentage of non-class members which have a larger 
	 distance to the considered in-class time series, averaged over all classes.
	 ED perfectly tolerates contamination, while $\mathcal{E}$ and EDR commonly, but not always achieve perfect scores. In terms of imprecision invariance, all measures besides ED appear to fulfill this property. Altogether, EDR with a median absolute deviation-based tolerance interval appears to have the highest combined robustness.}
	\label{tab:results21}
\end{table*}
\begin{table*}[!ht]
	\begin{tabular}{l|cccccc|cccccc}
		Dataset& \multicolumn{6}{c}{\textit{Contamination Tolerance}}&
		\multicolumn{6}{c}{\textit{Imprecision Invariance}}\\
		&$\mathcal{E}$&Euc&DTW&Log&ED&EDR&$\mathcal{E}$&Euc&DTW&Log&ED&EDR\\
		\toprule
		
		MedicalImages & 1.00 & 0.00 & 0.00 & 0.00 & 1.00 & 1.00 & 1.00 & 1.00 & 1.00 & 1.00 & 0.00 & 1.00 \\ 
		MiddlePhalanxOutlineAgeGroup & 0.69 & 0.00 & 0.00 & 0.00 & 1.00 & 1.00 & 1.00 & 1.00 & 1.00 & 1.00 & 0.00 & 1.00 \\ 
		MiddlePhalanxOutlineCorrect & 0.54 & 0.00 & 0.00 & 0.00 & 1.00 & 1.00 & 1.00 & 1.00 & 1.00 & 1.00 & 0.00 & 1.00 \\ 
		MiddlePhalanxTW & 0.84 & 0.00 & 0.00 & 0.00 & 1.00 & 1.00 & 1.00 & 1.00 & 1.00 & 1.00 & 0.00 & 1.00 \\ 
		OliveOil & 0.00 & 0.00 & 0.00 & 0.00 & 1.00 & 0.01 & 1.00 & 1.00 & 1.00 & 1.00 & 0.00 & 0.51 \\ 
		OSULeaf & 1.00 & 0.00 & 0.00 & 0.00 & 1.00 & 1.00 & 1.00 & 1.00 & 1.00 & 1.00 & 0.00 & 1.00 \\ 
		PigAirwayPressure & 1.00 & 0.00 & 0.00 & 0.00 & 1.00 & 1.00 & 1.00 & 1.00 & 1.00 & 1.00 & 0.00 & 1.00 \\ 
		PigArtPressure & 1.00 & 0.00 & 0.00 & 0.00 & 1.00 & 1.00 & 1.00 & 1.00 & 1.00 & 1.00 & 0.09 & 1.00 \\ 
		PigCVP & 1.00 & 0.00 & 0.00 & 0.00 & 1.00 & 1.00 & 1.00 & 1.00 & 1.00 & 1.00 & 0.00 & 1.00 \\ 
		Plane & 1.00 & 0.00 & 0.00 & 0.00 & 1.00 & 1.00 & 1.00 & 1.00 & 1.00 & 1.00 & 0.00 & 1.00 \\ 
		PowerCons & 1.00 & 0.00 & 0.00 & 0.00 & 1.00 & 1.00 & 1.00 & 1.00 & 1.00 & 1.00 & 0.00 & 1.00 \\ 
		ProximalPhalanxOutlineAgeGroup & 0.88 & 0.00 & 0.00 & 0.00 & 1.00 & 1.00 & 1.00 & 1.00 & 1.00 & 1.00 & 0.00 & 1.00 \\ 
		ProximalPhalanxOutlineCorrect & 0.52 & 0.00 & 0.00 & 0.00 & 1.00 & 1.00 & 1.00 & 1.00 & 1.00 & 1.00 & 0.00 & 1.00 \\ 
		ProximalPhalanxTW & 0.71 & 0.00 & 0.00 & 0.00 & 1.00 & 1.00 & 1.00 & 1.00 & 1.00 & 1.00 & 0.00 & 1.00 \\ 
		RefrigerationDevices & 0.97 & 0.00 & 0.00 & 0.00 & 1.00 & 1.00 & 1.00 & 1.00 & 1.00 & 1.00 & 0.00 & 1.00 \\ 
		Rock & 1.00 & 0.00 & 0.00 & 0.00 & 1.00 & 1.00 & 1.00 & 1.00 & 1.00 & 1.00 & 0.07 & 1.00 \\ 
		ScreenType & 1.00 & 0.00 & 0.00 & 0.00 & 1.00 & 1.00 & 1.00 & 1.00 & 1.00 & 1.00 & 0.00 & 1.00 \\ 
		SemgHandGenderCh2 & 1.00 & 0.00 & 0.00 & 0.00 & 1.00 & 1.00 & 1.00 & 1.00 & 1.00 & 1.00 & 0.02 & 1.00 \\ 
		SemgHandMovementCh2 & 1.00 & 0.00 & 0.00 & 0.00 & 1.00 & 1.00 & 1.00 & 1.00 & 1.00 & 1.00 & 0.02 & 1.00 \\ 
		SemgHandSubjectCh2 & 1.00 & 0.00 & 0.00 & 0.00 & 1.00 & 1.00 & 1.00 & 1.00 & 1.00 & 1.00 & 0.02 & 1.00 \\ 
		ShapeletSim & 1.00 & 0.00 & 0.00 & 0.00 & 1.00 & 1.00 & 1.00 & 1.00 & 1.00 & 1.00 & 0.00 & 1.00 \\ 
		ShapesAll & 1.00 & 0.00 & 0.00 & 0.00 & 1.00 & 1.00 & 1.00 & 1.00 & 1.00 & 1.00 & 0.00 & 1.00 \\ 
		SmallKitchenAppliances & 1.00 & 0.00 & 0.00 & 0.00 & 1.00 & 1.00 & 1.00 & 1.00 & 1.00 & 1.00 & 0.00 & 1.00 \\ 
		SmoothSubspace & 0.00 & 0.00 & 0.00 & 0.00 & 1.00 & 1.00 & 1.00 & 1.00 & 1.00 & 1.00 & 0.00 & 1.00 \\ 
		SonyAIBORobotSurface1 & 1.00 & 0.00 & 0.00 & 0.00 & 1.00 & 1.00 & 1.00 & 1.00 & 1.00 & 1.00 & 0.00 & 1.00 \\ 
		SonyAIBORobotSurface2 & 1.00 & 0.00 & 0.00 & 0.00 & 1.00 & 1.00 & 1.00 & 1.00 & 1.00 & 1.00 & 0.00 & 1.00 \\ 
		Strawberry & 0.50 & 0.00 & 0.00 & 0.00 & 1.00 & 1.00 & 1.00 & 1.00 & 1.00 & 1.00 & 0.00 & 1.00 \\ 
		SwedishLeaf & 1.00 & 0.00 & 0.00 & 0.00 & 1.00 & 1.00 & 1.00 & 1.00 & 1.00 & 1.00 & 0.00 & 1.00 \\ 
		Symbols & 1.00 & 0.00 & 0.00 & 0.00 & 1.00 & 1.00 & 1.00 & 1.00 & 1.00 & 1.00 & 0.00 & 1.00 \\ 
		SyntheticControl & 1.00 & 0.00 & 0.00 & 0.00 & 1.00 & 1.00 & 1.00 & 1.00 & 1.00 & 1.00 & 0.00 & 1.00 \\ 
		ToeSegmentation1 & 1.00 & 0.00 & 0.00 & 0.00 & 1.00 & 1.00 & 1.00 & 1.00 & 1.00 & 1.00 & 0.00 & 1.00 \\ 
		ToeSegmentation2 & 1.00 & 0.00 & 0.00 & 0.00 & 1.00 & 1.00 & 1.00 & 1.00 & 1.00 & 1.00 & 0.00 & 1.00 \\ 
		Trace & 1.00 & 0.00 & 0.00 & 0.00 & 1.00 & 1.00 & 1.00 & 1.00 & 1.00 & 1.00 & 0.00 & 1.00 \\ 
		UMD & 1.00 & 0.00 & 0.00 & 0.00 & 1.00 & 1.00 & 1.00 & 1.00 & 1.00 & 1.00 & 0.00 & 1.00 \\ 
		Wine & 0.00 & 0.00 & 0.00 & 0.00 & 1.00 & 1.00 & 1.00 & 1.00 & 1.00 & 1.00 & 0.00 & 1.00 \\ 
		WordSynonyms & 1.00 & 0.00 & 0.00 & 0.00 & 1.00 & 1.00 & 1.00 & 1.00 & 1.00 & 1.00 & 0.00 & 1.00 \\ 
		Worms & 1.00 & 0.00 & 0.00 & 0.00 & 1.00 & 1.00 & 1.00 & 1.00 & 1.00 & 1.00 & 0.00 & 1.00 \\ 
		WormsTwoClass & 1.00 & 0.00 & 0.00 & 0.00 & 1.00 & 1.00 & 1.00 & 1.00 & 1.00 & 1.00 & 0.00 & 1.00 \\

	\end{tabular}
	\caption{Contamination tolerance and imprecision invariance per distance function, Part 2.}
	\label{tab:results22}
\end{table*}
\end{document}